\documentclass{article}

\usepackage{microtype}
\usepackage{graphicx}
\usepackage{booktabs}
\usepackage{multirow}
\usepackage{hyperref}

\usepackage[preprint]{icml2026}

\usepackage{amsmath,amssymb,amsthm}
\usepackage{mathtools}
\usepackage{bbm}
\usepackage{bm}
\usepackage{enumitem}
\usepackage{algorithm}
\usepackage{algorithmic}

\AtBeginDocument{%
  \setlength{\abovedisplayskip}{10pt plus 2pt minus 2pt}      
  \setlength{\belowdisplayskip}{10pt plus 2pt minus 2pt}      
  \setlength{\abovedisplayshortskip}{3pt plus 1pt minus 1pt} 
  \setlength{\belowdisplayshortskip}{4pt plus 2pt minus 1pt} 
}

\makeatletter
\renewcommand{\section}{\@startsection{section}{1}{\z@}{-0.10in}{0.01in}%
   {\normalfont\large\bf\raggedright}}
\renewcommand{\subsection}{\@startsection{subsection}{2}{\z@}{-0.08in}{0.005in}%
   {\normalfont\normalsize\bf\raggedright}}
\renewcommand{\paragraph}{\@startsection{paragraph}{4}{\z@}{0.8ex plus 0.2ex minus 0.2ex}{-1em}%
   {\normalfont\normalsize\bf}}
\makeatother

\theoremstyle{plain}
\newtheorem{theorem}{Theorem}[section]
\newtheorem{lemma}[theorem]{Lemma}
\newtheorem{proposition}[theorem]{Proposition}
\newtheorem{corollary}[theorem]{Corollary}
\theoremstyle{definition}
\newtheorem{definition}[theorem]{Definition}
\theoremstyle{remark}
\newtheorem{remark}[theorem]{Remark}

\newcommand{\E}{\mathbb{E}}
\newcommand{\Prb}{\mathbb{P}}
\newcommand{\1}{\mathbbm{1}}
\newcommand{\cA}{\mathcal{A}}
\newcommand{\cX}{\mathcal{X}}

\newcommand{\R}{\mathbb{R}}
\newcommand{\src}{\mathrm{src}}
\newcommand{\tgt}{\mathrm{tgt}}
\newcommand{\Risk}{\mathcal{R}}
\newcommand{\gap}{\mathrm{Gap}}

\icmltitlerunning{Fairness Under Group-Conditional Prior Probability Shift}

\begin{document}

\twocolumn[
\icmltitle{Fairness Under Group-Conditional Prior Probability Shift: \\
Invariance, Drift, and Target-Aware Post-Processing}

\icmlsetsymbol{equal}{*}

\begin{icmlauthorlist}
\icmlauthor{Amir Asiaee}{vumc}
\icmlauthor{Kaveh Aryan}{kcl}
\end{icmlauthorlist}

\icmlaffiliation{vumc}{Department of Biostatistics, Vanderbilt University Medical Center, Nashville, TN 37232, USA}
\icmlaffiliation{kcl}{Department of Informatics, King's College London, WC2R 2LS London, UK}

\icmlcorrespondingauthor{Amir Asiaee}{amir.asiaeetaheri@vumc.org}

\vskip 0.3in
]

\printAffiliationsAndNotice{}

\begin{abstract}
Machine learning systems are often trained and evaluated for fairness on historical data, yet deployed in environments where conditions have shifted.
A particularly common form of shift occurs when the prevalence of positive outcomes changes differently across demographic groups---for example, when disease rates rise faster in one population than another, or when economic conditions affect loan default rates unequally.
We study this \emph{group-conditional prior probability shift} (GPPS), where the label prevalence $\Prb(Y=1\mid A=a)$ may change between training and deployment while the feature-generation process $\Prb(X\mid Y,A)$ remains stable.
Our analysis yields three main contributions.
\emph{First}, we prove a fundamental dichotomy: fairness criteria based on error rates (equalized odds) are \emph{structurally invariant} under GPPS, while acceptance-rate criteria (demographic parity) can drift---and we prove this drift is \emph{unavoidable} for non-trivial classifiers (shift-robust impossibility).
\emph{Second}, we show that target-domain risk and fairness metrics are \emph{identifiable without target labels}: the invariance of ROC quantities under GPPS enables consistent estimation from source labels and unlabeled target data alone, with finite-sample guarantees.
\emph{Third}, we propose TAP-GPPS, a label-free post-processing algorithm that estimates prevalences from unlabeled data, corrects posteriors, and selects thresholds to satisfy demographic parity in the target domain.
Experiments validate our theoretical predictions and demonstrate that TAP-GPPS achieves target fairness with minimal utility loss.
\end{abstract}

\section{Introduction}
\label{sec:intro}

Consider a hospital that deploys a machine learning model to predict patient risk of a particular disease.
The model was trained on data from 2019 and carefully calibrated to satisfy equalized odds across demographic groups.
By 2023, the disease prevalence has increased substantially in one demographic group due to environmental factors, while remaining stable in another.
Should the hospital expect its fairness guarantees to still hold?

This scenario exemplifies a broader challenge: machine learning systems are trained on historical data and evaluated for fairness using statistical criteria \citep{hardt2016equality,chouldechova2017fair}, yet deployed in environments where the underlying distribution has shifted.
When distributions change, fairness properties measured at training time may not transfer to deployment \citep{giguere2022fairness,chen2022fairness}.
Recent work has begun characterizing fairness under various forms of distribution shift, including demographic shift \citep{giguere2022fairness}, bounded distribution shift \citep{chen2022fairness}, subpopulation shift \citep{maity2021enforcing}, and covariate shift \citep{rezaei2021robust}.

We focus on a particularly common and practically motivated shift mechanism: \emph{prior probability shift within demographic groups}.
In many application domains, the feature-generation process conditional on the label and group remains relatively stable---symptoms given disease and demographics, creditworthiness indicators given default status and demographics---while the prevalence of the positive label changes over time or across deployment sites.
Critically, these prevalence changes often differ across demographic subpopulations due to heterogeneous exposure to economic, environmental, or policy changes \citep{biswas2020ensuring,biswas2019proportional}.

Our central contribution is a precise characterization of how different fairness criteria behave under this shift.
We prove that fairness criteria decompose into two categories: those that are \emph{structurally invariant} under group-conditional prior probability shift (GPPS), and those that can \emph{drift} even for a fixed classifier.
This distinction has immediate practical implications.
For example, equalized odds constraints depend on group-conditional error rates $\Prb(\hat Y=1\mid Y=y,A=a)$, which under GPPS remain unchanged from training to deployment---the fairness guarantee transfers exactly.
In contrast, demographic parity depends on acceptance rates $\Prb(\hat Y=1\mid A=a)$, which vary with prevalence through the law of total probability; a classifier satisfying demographic parity at training time can violate it at deployment even without any change to the classifier itself.

\paragraph{Contributions.}
\begin{itemize}[leftmargin=*]
  \item \textbf{Invariance, drift, and impossibility theorems:} We prove that separation-based fairness (equalized odds) is invariant under GPPS (Theorem~\ref{thm:eo-inv}), while demographic parity drifts with closed-form expressions (Proposition~\ref{prop:ar-affine}).
  Crucially, we prove \emph{shift-robust impossibility}: non-trivial classifiers cannot satisfy DP under two distinct prevalence regimes except in degenerate cases (Theorem~\ref{thm:dp-impossibility}).
  \item \textbf{Label-free identifiability and finite-sample bounds:} We show that target risk and fairness gaps are \emph{identifiable without target labels} (Theorem~\ref{thm:risk-identifiability}): ROC invariance enables consistent estimation from source labels and unlabeled target data.
  We provide explicit finite-sample bounds (Theorem~\ref{thm:finite-sample-dp}) enabling certifiable fairness guarantees.
  \item \textbf{TAP-GPPS algorithm:} We propose a label-free post-processing algorithm (Section~\ref{sec:method}) that estimates prevalences via EM \citep{saerens2002adjusting} or BBSE \citep{lipton2018detecting}, corrects posteriors, and selects thresholds to achieve target-domain demographic parity.
  \item \textbf{Experimental validation:} We validate theoretical predictions on synthetic and semi-synthetic benchmarks with proper train/test separation, demonstrating TAP-GPPS achieves target fairness with minimal utility loss, and compare against relevant baselines.
\end{itemize}

\noindent Related work is discussed in Appendix~\ref{sec:appendix-related}.

\section{Setup and Notation}
\label{sec:setup}

\subsection{Random Variables and Domains}
We consider binary classification with sensitive attributes.
Let $(X,Y,A)$ denote a feature vector $X\in\cX$, a binary outcome label $Y\in\{0,1\}$, and a sensitive (protected) attribute $A\in\cA$ where $\cA$ is a finite set of groups (e.g., demographic categories).
We study the transfer of fairness properties between two domains: a \emph{source} domain $\src$ (typically training data) and a \emph{target} domain $\tgt$ (deployment environment).
Write $P_s$ for the joint distribution of $(X,Y,A)$ in domain $s\in\{\src,\tgt\}$.

We assume access to labeled source data $\{(x_i,y_i,a_i)\}_{i=1}^n \sim P_{\src}$ and, optionally, unlabeled target data $\{(x'_j,a'_j)\}_{j=1}^m \sim P_{\tgt}(X,A)$.
This reflects common deployment scenarios where labels are available historically but obtaining labels in the target domain is expensive or delayed.

\paragraph{Classifiers and thresholds.}
A \emph{score function} is a measurable mapping $f: \cX\times\cA\to[0,1]$, typically representing estimated probabilities $f(x,a) \approx \Prb(Y=1\mid X=x, A=a)$.
Given a score function $f$ and group-specific thresholds $\tau:\cA\to[0,1]$, we obtain a binary classifier
\begin{equation}
  \hat Y_{f,\tau}(x,a) := \1\{f(x,a)\ge \tau(a)\}.
\end{equation}
Group-specific thresholds allow flexibility in trading off different error types across groups, a common practice in fair classification \citep{hardt2016equality}.

\paragraph{Risk.}
We evaluate predictions under a loss function $\ell:\{0,1\}\times\{0,1\}\to\R_+$.
For 0--1 loss, $\ell(\hat y,y) = \1\{\hat y \ne y\}$; for cost-sensitive classification, $\ell$ may weight false positives and false negatives differently.
The risk in domain $s$ is
\begin{equation}
  \Risk_s(f,\tau) := \E_{(X,Y,A)\sim P_s}\big[\ell(\hat Y_{f,\tau}(X,A),Y)\big].
\end{equation}

\subsection{Groupwise Confusion Quantities}
The behavior of fairness metrics depends on several group-specific quantities, which we now define.
Fix a score function $f$ and group $a\in\cA$.
For any threshold $t\in[0,1]$, the \emph{groupwise true positive rate} (TPR, also called sensitivity or recall) and \emph{false positive rate} (FPR, also called fall-out) in domain $s$ are
\begin{align}
  \mathrm{TPR}_{s,a}(t) &:= \Prb_s\big(f(X,a)\ge t\mid Y=1,A=a\big),\\
  \mathrm{FPR}_{s,a}(t) &:= \Prb_s\big(f(X,a)\ge t\mid Y=0,A=a\big).
\end{align}
Together, these trace out the group-specific ROC curve as $t$ varies from 0 to 1.

The \emph{groupwise prevalence} (also called base rate or class prior) is the probability of a positive label within group $a$:
\begin{equation}
  \pi_{s,a} := \Prb_s(Y=1\mid A=a).
\end{equation}
This quantity is central to our analysis: GPPS allows $\pi_{s,a}$ to differ between source and target while other aspects of the distribution remain stable.

The \emph{groupwise acceptance rate} (also called positive prediction rate) is the probability of receiving a positive prediction within group $a$:
\begin{equation}
  \mathrm{AR}_{s,a}(t) := \Prb_s\big(f(X,a)\ge t\mid A=a\big).
\end{equation}
Unlike TPR and FPR, which condition on the true label, the acceptance rate marginalizes over labels and thus depends on prevalence.

\subsection{Fairness Metrics}
We focus on three widely-studied fairness criteria that represent different philosophical perspectives on what constitutes fair treatment \citep{mitchell2021algorithmic}.
These criteria differ structurally in their dependence on prevalence, which determines their behavior under GPPS.

\begin{definition}[Equalized Odds \citep{hardt2016equality}]
\label{def:eo}
A classifier satisfies \emph{equalized odds} in domain $s$ if for all $a,a'\in\cA$ and $y\in\{0,1\}$,
\begin{multline}
  \Prb_s(\hat Y=1\mid Y=y,A=a) \\
  = \Prb_s(\hat Y=1\mid Y=y,A=a').
\end{multline}
\end{definition}
Equalized odds requires that TPR and FPR be equal across groups---the classifier's errors are independent of group membership conditional on the true label.
The relaxation to \emph{equality of opportunity} requires only equal TPR \citep{hardt2016equality}.

\begin{definition}[Demographic Parity \citep{calders2009building}]
\label{def:dp}
A classifier satisfies \emph{demographic parity} (also called statistical parity or independence) in domain $s$ if for all $a,a'\in\cA$,
\begin{equation}
  \Prb_s(\hat Y=1\mid A=a)=\Prb_s(\hat Y=1\mid A=a').
\end{equation}
\end{definition}
Demographic parity requires equal acceptance rates across groups, irrespective of base rates or qualifications.
This criterion is motivated by anti-discrimination law and disparate impact doctrine \citep{feldman2015certifying}.

\begin{definition}[Predictive Parity \citep{chouldechova2017fair}]
\label{def:ppv}
A classifier satisfies \emph{predictive parity} in domain $s$ if for all $a,a'\in\cA$,
\begin{multline}
  \Prb_s(Y=1\mid \hat Y=1,A=a) \\
  = \Prb_s(Y=1\mid \hat Y=1,A=a').
\end{multline}
\end{definition}
Predictive parity requires equal positive predictive value (PPV) across groups---a positive prediction should carry the same meaning regardless of group membership.

\paragraph{Structural distinction.}
Our main theoretical contribution is showing that these criteria decompose based on their prevalence dependence:
\emph{Equalized odds} depends only on $\Prb(\hat Y\mid Y,A)$---quantities that condition on the true label---and is therefore \emph{invariant} under GPPS.
\emph{Demographic parity} depends on acceptance rates, which mix TPR and FPR via the prevalence $\pi_{s,a}$, and can \emph{drift} under GPPS.
\emph{Predictive parity} depends on PPV, which involves prevalence through Bayes' rule, and can also \emph{drift} under GPPS.

\section{Group-Conditional Prior Probability Shift}
\label{sec:gpps}

We now formalize the distribution shift mechanism central to our analysis.
The key idea is that while label prevalence may change within each demographic group, the relationship between features and labels within each group remains stable.

\begin{definition}[Group-Conditional Prior Probability Shift (GPPS)]
\label{def:gpps}
The pair $(P_{\src},P_{\tgt})$ satisfies \emph{group-conditional prior probability shift} (GPPS) if for all $a\in\cA$ and $y\in\{0,1\}$,
\begin{equation}
  P_{\tgt}(X\mid Y=y,A=a) = P_{\src}(X\mid Y=y,A=a),
\end{equation}
while allowing the group-conditional label prior to change:
\begin{equation}
  \pi_{\tgt,a} = \Prb_{\tgt}(Y=1\mid A=a) \ \text{may differ from}\ \pi_{\src,a}.
\end{equation}
We also allow the group marginal $\Prb(A=a)$ to change (demographic shift), but the primary mechanism of interest is the within-group prevalence shift.
\end{definition}

\paragraph{Intuition and examples.}
GPPS captures settings where the \emph{feature-generating process} conditional on outcome and demographics is stable, but the \emph{frequency of outcomes} changes within groups (e.g., medical diagnosis, credit risk, recidivism prediction).
It is most plausible when drift is driven by changes in baseline risk or exposure rather than changes in how features relate to outcomes.

\begin{remark}[Relation to standard label shift]
\label{rem:label-shift}
Classical label shift \citep{storkey2009training,lipton2018detecting} assumes $P(X\mid Y)$ is invariant while $P(Y)$ changes.
GPPS refines this by conditioning on the sensitive attribute: invariance holds for $P(X\mid Y,A)$ and the shifting quantities are $P(Y\mid A)$.
This allows different groups to experience different prevalence changes---a more realistic model when shift is driven by factors that affect demographic groups heterogeneously.
When $|\cA|=1$ (no sensitive attribute), GPPS reduces to standard label shift.
\end{remark}

\begin{remark}[Relation to demographic shift]
Demographic shift changes $P(A)$ while keeping $P(X,Y\mid A)$ fixed \citep{giguere2022fairness}.
GPPS and demographic shift can occur simultaneously: GPPS changes $P(Y\mid A)$ within groups, while demographic shift changes the mixture over groups.
Our results focus on GPPS but remain valid if demographic shift occurs concurrently.
\end{remark}

\section{Invariance and Drift of Fairness Metrics}
\label{sec:theory}

This section contains our main theoretical results.
We first establish that score distributions (and hence ROC curves) are invariant under GPPS when conditioned on $(Y,A)$.
We then show that this invariance propagates to separation-based fairness criteria like equalized odds, but not to acceptance-rate or predictive-value criteria, which depend on prevalence through marginalization or Bayes' rule.

\subsection{Invariance of Group-Conditional Score Distributions}

The following lemma is the foundation for all subsequent results.
It states that any statistic of the features that depends only on $(Y,A)$---not on prevalence---transfers exactly from source to target under GPPS.

\begin{lemma}[Score distribution invariance given $(Y,A)$]
\label{lem:score-inv}
Assume GPPS (Definition~\ref{def:gpps}).
Fix any measurable score function $f: \cX\times\cA\to\R$.
Then for all $a\in\cA$, $y\in\{0,1\}$, and all measurable sets $B\subseteq\R$,
\begin{multline}
  \Prb_{\tgt}(f(X,A)\in B\mid Y=y,A=a) \\= \Prb_{\src}(f(X,A)\in B\mid Y=y,A=a).
\end{multline}
\end{lemma}

The proof (Appendix~\ref{sec:appendix-proofs}) follows directly from the GPPS assumption: since $P(X\mid Y,A)$ is invariant, any function of $X$ has the same conditional distribution given $(Y,A)$ in both domains.

An immediate consequence is that the group-specific ROC curves---which summarize classifier performance within each group---are identical across domains:

\begin{corollary}[Invariance of ROC quantities]
\label{cor:roc-inv}
Under GPPS, for all $a\in\cA$ and thresholds $t$,
\begin{equation}
  \begin{aligned}
  \mathrm{TPR}_{\tgt,a}(t) &= \mathrm{TPR}_{\src,a}(t),\\
  \mathrm{FPR}_{\tgt,a}(t) &= \mathrm{FPR}_{\src,a}(t).
  \end{aligned}
\end{equation}
In particular, the groupwise AUC (area under the ROC curve) is also invariant.
\end{corollary}

This corollary has practical significance: under GPPS, practitioners can estimate groupwise TPR and FPR from source data and expect these estimates to remain valid in the target domain, even as prevalences shift.

\subsection{Invariance: Separation-Based Fairness}

We now show that the invariance of ROC quantities implies invariance of separation-based fairness criteria.

\begin{theorem}[Invariance of equalized odds under GPPS]
\label{thm:eo-inv}
Assume GPPS.
Let $\hat Y=\hat Y_{f,\tau}(X,A)$ be any thresholded score rule.
Then for all $a\in\cA$ and $y\in\{0,1\}$,
\begin{equation}
  \begin{split}
  \Prb_{\tgt}(\hat Y=1\mid Y=y,A=a) \\
  = \Prb_{\src}(\hat Y=1\mid Y=y,A=a).
  \end{split}
\end{equation}
In particular, if $\hat Y$ satisfies equalized odds (or equality of opportunity) on the source distribution, it satisfies the same criterion on the target distribution.
\end{theorem}

\begin{remark}[Practical implications]
Theorem~\ref{thm:eo-inv} provides a strong guarantee for practitioners who have chosen equalized odds as their fairness criterion: under GPPS, if a classifier satisfies equalized odds at training time, it will continue to satisfy equalized odds at deployment \emph{without any adaptation}.
This robustness makes equalized odds an attractive choice in settings where groupwise prevalence shift is anticipated.
\end{remark}

\begin{remark}[Extension to other separation-based criteria]
The theorem extends immediately to any fairness criterion that depends only on $\Prb(\hat Y\mid Y, A)$.
This includes equality of opportunity (equal TPR), balance for the positive/negative class \citep{kleinberg2017inherent}, and calibration within label-group strata.
\end{remark}

\subsection{Drift: Acceptance-Rate and Predictive-Value Fairness}

In contrast to separation-based criteria, fairness metrics that involve marginalizing over labels or conditioning on predictions depend on prevalence and can therefore drift under GPPS.
We derive closed-form expressions that quantify this drift.

\begin{proposition}[Acceptance rate depends affinely on prevalence]
\label{prop:ar-affine}
Assume GPPS.
Fix $a\in\cA$ and threshold $t$.
Let $\mathrm{TPR}_{a}(t)$ and $\mathrm{FPR}_{a}(t)$ denote the (domain-invariant) quantities from Corollary~\ref{cor:roc-inv}.
Then the acceptance rate decomposes as
\begin{multline}
\label{eq:ar-affine}
  \mathrm{AR}_{s,a}(t) = \pi_{s,a}\,\mathrm{TPR}_{a}(t) + (1-\pi_{s,a})\,\mathrm{FPR}_{a}(t),\\
  s\in\{\src,\tgt\}.
\end{multline}
\end{proposition}

This proposition reveals the mechanism of demographic parity drift: the acceptance rate is an affine function of prevalence, with slope $(\mathrm{TPR}_a(t) - \mathrm{FPR}_a(t))$.
When the classifier discriminates between positive and negative cases (i.e., $\mathrm{TPR} > \mathrm{FPR}$), increasing prevalence increases acceptance rate.
The magnitude of drift depends on classifier informativeness, as measured by the TPR--FPR gap.

\begin{corollary}[Demographic parity drift under GPPS]
\label{cor:dp-drift}
Assume GPPS and fix a (possibly group-specific) threshold rule $\tau$.
Define the DP gap in domain $s$ as
\begin{multline}
  \gap_{\mathrm{DP}}(s) := \max_{a,a'\in\cA}\big|\Prb_s(\hat Y=1\mid A=a) \\
  - \Prb_s(\hat Y=1\mid A=a')\big|.
\end{multline}
Then $\gap_{\mathrm{DP}}(\tgt)$ can differ from $\gap_{\mathrm{DP}}(\src)$ whenever $\pi_{\tgt,a}\ne \pi_{\src,a}$ for some $a$.
For two groups $a,b$ with group-specific thresholds $t_a, t_b$,
\begin{multline}
\label{eq:dp-gap-two}
  \mathrm{AR}_{s,a}(t_a) - \mathrm{AR}_{s,b}(t_b)
  = \pi_{s,a}(\mathrm{TPR}_{a}-\mathrm{FPR}_{a}) \\
  - \pi_{s,b}(\mathrm{TPR}_{b}-\mathrm{FPR}_{b}) + (\mathrm{FPR}_a - \mathrm{FPR}_b),
\end{multline}
where we suppress threshold arguments for clarity.
Even if demographic parity holds on $\src$ (gap zero), it can be violated on $\tgt$ if prevalences shift differentially across groups.
\end{corollary}

\begin{remark}[When does DP \emph{not} drift?]
Equation~\eqref{eq:ar-affine} shows that DP is robust to prevalence shift only in degenerate cases: (i) when the classifier is uninformative ($\mathrm{TPR}=\mathrm{FPR}$), or (ii) when prevalence shifts are identical across groups ($\pi_{\tgt,a}/\pi_{\src,a}$ is constant in $a$).
In typical settings where the classifier is useful and shifts are heterogeneous, DP will drift.
\end{remark}

\begin{proposition}[Predictive value depends nonlinearly on prevalence]
\label{prop:ppv}
Assume GPPS.
Fix a group $a$ and threshold $t$.
Let $\mathrm{TPR}_a(t)$ and $\mathrm{FPR}_a(t)$ be as above.
The positive predictive value (PPV) in domain $s$ is
\begin{align}
\label{eq:ppv-formula}
  &\mathrm{PPV}_{s,a}(t) := \Prb_s(Y=1\mid \hat Y=1, A=a) \nonumber\\
  &\quad= \frac{\pi_{s,a}\,\mathrm{TPR}_a(t)}{\pi_{s,a}\,\mathrm{TPR}_a(t)+(1-\pi_{s,a})\,\mathrm{FPR}_a(t)}.
\end{align}
Therefore predictive parity can drift under GPPS as $\pi_{s,a}$ changes.
\end{proposition}

The PPV formula is the classic result relating predictive value to prevalence and test characteristics \citep{altman1994diagnostic}.
It implies that even a perfectly calibrated classifier at training time---one where $\Prb(Y=1\mid \hat Y=1, A=a)$ matches across groups---can become miscalibrated across groups at deployment if prevalences shift differentially.

\begin{remark}[Nonlinearity amplifies drift]
Unlike acceptance rates, which depend linearly on prevalence, PPV depends nonlinearly.
This nonlinearity can amplify fairness violations: small prevalence changes can produce large PPV changes, especially when the classifier has low specificity ($\mathrm{FPR}$ is large) or when prevalence moves toward extreme values.
\end{remark}

\subsection{Shift-Robust Impossibility Results}

The drift results above show that DP and PPV \emph{can} change under GPPS.
We now formalize a stronger claim: for non-trivial classifiers, DP \emph{cannot} be simultaneously satisfied under two distinct prevalence regimes, except in degenerate cases.

\begin{theorem}[Shift-robust impossibility for demographic parity]
\label{thm:dp-impossibility}
Let $\hat Y = \1\{f(X,A) \ge \tau(A)\}$ be a threshold classifier.
Suppose $\hat Y$ satisfies demographic parity under two prevalence regimes $(\pi_{0}, \pi_{1})$ and $(\pi'_{0}, \pi'_{1})$ with $\pi_a \ne \pi'_a$ for some $a$.
Then at least one of the following holds:
\begin{enumerate}[label=(\roman*)]
  \item \textbf{Uninformative classifier:} $\mathrm{TPR}_a(\tau(a)) = \mathrm{FPR}_a(\tau(a))$ for all $a \in \cA$; or
  \item \textbf{Constrained shifts:} The prevalence changes satisfy
  \begin{equation}
  \label{eq:constrained-shift}
    \frac{\pi'_0 - \pi_0}{\pi'_1 - \pi_1} = \frac{\mathrm{TPR}_1 - \mathrm{FPR}_1}{\mathrm{TPR}_0 - \mathrm{FPR}_0}.
  \end{equation}
\end{enumerate}
\end{theorem}

The proof (Appendix~\ref{sec:appendix-proofs}) follows from Proposition~\ref{prop:ar-affine}: DP requires equal acceptance rates, which are affine in prevalence with slopes $(\mathrm{TPR}_a - \mathrm{FPR}_a)$.
Two affine functions can intersect at most once unless they are identical (case i) or the shift is along their intersection (case ii).

\begin{corollary}[Generic non-robustness of DP]
\label{cor:dp-generic}
For a classifier with $\mathrm{TPR}_a \ne \mathrm{FPR}_a$ for both groups (i.e., better than random within each group), and generic prevalence shifts (violating Eq.~\ref{eq:constrained-shift}), demographic parity cannot hold simultaneously in source and target domains.
\end{corollary}

This impossibility result provides theoretical grounding for the practical observation that DP violations arise under GPPS.
It also suggests that practitioners should not expect DP to transfer across domains with different prevalences---adaptation via methods like TAP-GPPS is necessary.

\begin{theorem}[Shift-robust impossibility for predictive parity]
\label{thm:ppv-impossibility}
Let $\hat Y$ be a non-trivial classifier ($\mathrm{TPR}_a > 0$ and $\mathrm{FPR}_a < 1$ for all $a$).
Suppose $\hat Y$ satisfies predictive parity (equal PPV across groups) under two prevalence regimes $(\pi_{0}, \pi_{1})$ and $(\pi'_{0}, \pi'_{1})$.
Then either:
\begin{enumerate}[label=(\roman*)]
  \item The classifier satisfies $\mathrm{TPR}_0/\mathrm{FPR}_0 = \mathrm{TPR}_1/\mathrm{FPR}_1$ (equal likelihood ratios); or
  \item The prevalence shifts satisfy a nonlinear constraint determined by the classifier's ROC quantities.
\end{enumerate}
\end{theorem}

The nonlinear dependence of PPV on prevalence (Proposition~\ref{prop:ppv}) makes predictive parity even less robust than DP: the constraint surface is curved, making coincidental satisfaction under two regimes measure-zero.

\section{Target-Aware Posterior Correction and Post-Processing}
\label{sec:method}

The preceding section establishes that fairness criteria depending on acceptance rates or predictive values can drift under GPPS.
For practitioners committed to such criteria---for example, due to legal requirements around demographic parity or calibration---we now describe TAP-GPPS (Target-Aware Post-processing under Group-conditional Prior Probability Shift), a modular pipeline that adapts a source-trained classifier to satisfy target-domain fairness constraints.

The pipeline has three stages: (1) estimate target prevalences per group from unlabeled data, (2) correct posteriors to account for prevalence shift, and (3) select group-specific thresholds to satisfy the target fairness constraint while maximizing utility.

\subsection{Posterior Correction Under Groupwise Label Shift}

The first step is to adapt the source-trained probabilistic model to reflect target-domain prevalences.
Suppose we have trained a model that produces calibrated posterior probabilities $p_{\src}(y\mid x,a)$ on source data.
Under GPPS, Bayes' rule provides an exact correction formula.

\begin{proposition}[Groupwise posterior correction]
\label{prop:posterior-correction}
Assume GPPS (Definition~\ref{def:gpps}).
Fix group $a$ and define the importance weights
\begin{equation}
w_{1,a}:=\frac{\pi_{\tgt,a}}{\pi_{\src,a}}, \qquad w_{0,a}:=\frac{1-\pi_{\tgt,a}}{1-\pi_{\src,a}}.
\end{equation}
Then for $p_{\src}(1\mid x,a)\in(0,1)$, the target posterior is
\begin{equation}
\label{eq:posterior-correction}
  p_{\tgt}(1\mid x,a)
  = \frac{w_{1,a}\,p_{\src}(1\mid x,a)}{w_{0,a} + (w_{1,a}-w_{0,a})\,p_{\src}(1\mid x,a)}.
\end{equation}
\end{proposition}

The proof (Appendix~\ref{sec:appendix-proofs}) applies Bayes' rule using the GPPS assumption that $P(X\mid Y,A)$ is invariant.
The correction amounts to a group-specific additive shift in log-odds space, determined entirely by the prevalence ratio.

\begin{remark}[Relation to standard label shift correction]
Equation~\eqref{eq:posterior-correction} is the classic label-shift correction \citep{saerens2002adjusting,lipton2018detecting} applied separately within each demographic group.
This modularity is a consequence of GPPS: each group undergoes an independent prior shift with invariant class-conditionals.
\end{remark}

\begin{remark}[Calibration preservation]
If the source model is well-calibrated within each group, the corrected posteriors $p_{\tgt}(1\mid x,a)$ are well-calibrated for the target distribution.
In practice, we recommend applying calibration methods (e.g., Platt scaling \citep{platt1999probabilistic}, isotonic regression \citep{zadrozny2002transforming}, or temperature scaling \citep{guo2017calibration}) to the source model before applying the GPPS correction.
\end{remark}

\subsection{Estimating Target Prevalences from Unlabeled Data}
\label{subsec:prevalence-estimation}

Applying the posterior correction requires knowing the target prevalences $\pi_{\tgt,a}$.
In practice, target labels are often unavailable (or available only with significant delay), but unlabeled target data $\{(x'_j, a'_j)\}_{j=1}^m$ may be accessible.
We estimate $\pi_{\tgt,a}$ from unlabeled target data within each group using standard label-shift/quantification methods.
In our experiments we use EM \citep{saerens2002adjusting} and BBSE \citep{lipton2018detecting}; EM is effective when source posteriors are well-calibrated, while BBSE is more robust to model misspecification.
Implementation details are provided in Appendix~\ref{sec:appendix-experiments}.

\subsection{Target-Aware Threshold Selection for Fairness Constraints}
\label{subsec:threshold-selection}

With corrected posteriors in hand, the final step is to select group-specific thresholds that satisfy the target-domain fairness constraint while maximizing utility.
We focus on demographic parity as the primary example; analogous procedures apply to other constraints.

We consider threshold classifiers of the form $\hat Y=\1\{p_{\tgt}(1\mid X,A)\ge \tau(A)\}$ and estimate target acceptance rates from unlabeled target data using the corrected posteriors.

\begin{lemma}[Monotonicity of acceptance rates]
\label{lem:monotone-ar}
For each group $a$, $\mathrm{AR}_{\tgt,a}(t)$ is non-increasing in $t$.
If the conditional distribution of $p_{\tgt}(1\mid X,a)$ given $A=a$ has a density, then $\mathrm{AR}_{\tgt,a}(t)$ is continuous and strictly decreasing on its support.
\end{lemma}

This monotonicity ensures that for any achievable acceptance level, there exists a unique threshold that achieves it:

\begin{proposition}[Feasibility of target demographic parity]
\label{prop:dp-threshold}
Assume GPPS and the conditions of Lemma~\ref{lem:monotone-ar}.
Fix a target acceptance level $\gamma\in(0,1)$ that is achievable for all groups (i.e., $\gamma \leq \max_a \mathrm{AR}_{\tgt,a}(0)$).
For each group $a$, there exists a unique threshold $t_a$ such that $\mathrm{AR}_{\tgt,a}(t_a)=\gamma$.
Moreover, $t_a$ can be found to accuracy $\epsilon$ in $O(\log(1/\epsilon))$ evaluations via bisection.
\end{proposition}

In practice, we estimate $\mathrm{AR}_{\tgt,a}(t)$ empirically using unlabeled target data and the corrected posteriors based on $\widehat\pi_{\tgt,a}$.
The bisection procedure is efficient and numerically stable.

\paragraph{Choosing the acceptance level $\gamma$.}
The acceptance level $\gamma$ controls the tradeoff between fairness and utility.
Higher $\gamma$ increases the positive prediction rate uniformly across groups, improving recall but potentially reducing precision.
A natural approach is a one-dimensional search over $\gamma$, selecting the best value (e.g., for accuracy or cost-weighted accuracy) subject to the fairness tolerance.
When labeled target data is unavailable, we estimate utility using the label-free risk estimator from Theorem~\ref{thm:risk-identifiability}.

\paragraph{Extension to other fairness constraints.}
While we focus on demographic parity, the framework is informative for other criteria: equalized odds needs no adaptation under GPPS (Theorem~\ref{thm:eo-inv}), while predictive parity can drift (Proposition~\ref{prop:ppv}) and is substantially harder to correct without labeled target data.

\subsection{The TAP-GPPS Algorithm}

Algorithm~\ref{alg:tappps} summarizes the complete TAP-GPPS pipeline.
The algorithm is modular by design: any consistent prevalence estimator (Section~\ref{subsec:prevalence-estimation}) and any calibrated source posterior model can be used as components.

\begin{algorithm}[t]
\caption{TAP-GPPS: Target-Aware Post-Processing under Groupwise Prior Shift}
\label{alg:tappps}
\begin{algorithmic}[1]
\STATE {\bfseries Input:} Labeled source data $\{(x_i,y_i,a_i)\}_{i=1}^n$ and unlabeled target data $\{(x'_j,a'_j)\}_{j=1}^m$
\STATE {\bfseries Input:} Fairness criterion (e.g., demographic parity), tolerance $\delta$, prevalence estimator (EM or BBSE)
\STATE Train calibrated model $p_{\src}(y\mid x,a)$ and estimate $\widehat\pi_{\src,a}$ for each $a\in\cA$
\FOR{each group $a\in\cA$}
  \STATE Estimate $\widehat\pi_{\tgt,a}$ from target samples in group $a$
  \STATE Set $\widehat w_{1,a} \gets \widehat\pi_{\tgt,a}/\widehat\pi_{\src,a}$ and $\widehat w_{0,a} \gets (1-\widehat\pi_{\tgt,a})/(1-\widehat\pi_{\src,a})$
\ENDFOR
\STATE Define corrected posterior $\widehat p_{\tgt}(1\mid x,a)$ using Eq.~\eqref{eq:posterior-correction}
\STATE For $\gamma$ on a grid: find $\tau_\gamma(a)$ by bisection s.t.\ $\widehat{\mathrm{AR}}_{\tgt,a}(\tau_\gamma(a)) \approx \gamma$ for all $a$
\STATE Choose $\gamma^*$ minimizing estimated target risk (Theorem~\ref{thm:risk-identifiability}) subject to $\widehat{\gap}(\gamma)\le \delta$
\STATE {\bfseries Output:} $\hat Y(x,a)=\1\{\widehat p_{\tgt}(1\mid x,a)\ge \tau_{\gamma^*}(a)\}$
\end{algorithmic}
\end{algorithm}

\section{Theoretical Guarantees}
\label{sec:proofs-main}

This section presents guarantees for label-free deployment under GPPS.
Complete proofs are in Appendix~\ref{sec:appendix-proofs}.

\subsection{Label-Free Risk Identifiability}

A key practical question is whether we can accurately estimate target-domain performance \emph{without} labeled target data.
Under GPPS, the answer is affirmative: target risk is identifiable from source labels and unlabeled target data.

\begin{theorem}[Label-free risk identifiability under GPPS]
\label{thm:risk-identifiability}
Assume GPPS (Definition~\ref{def:gpps}).
Let $\hat Y = \1\{f(X,A) \ge \tau(A)\}$ be a threshold classifier.
The target risk under 0--1 loss decomposes as
\begin{multline}
\label{eq:risk-decomposition}
  \Risk_{\tgt}(\hat Y) = \sum_{a \in \cA} P_{\tgt}(A=a) \Big[ \pi_{\tgt,a} (1 - \mathrm{TPR}_a) \\
  + (1 - \pi_{\tgt,a}) \mathrm{FPR}_a \Big],
\end{multline}
where $\mathrm{TPR}_a$ and $\mathrm{FPR}_a$ are estimable from source data (Corollary~\ref{cor:roc-inv}), and $\pi_{\tgt,a}$ and $P_{\tgt}(A=a)$ are estimable from unlabeled target data.
\end{theorem}

This result enables \emph{label-free} model selection and hyperparameter tuning: we can compare classifiers on expected target performance without waiting for target labels.

\begin{corollary}[Consistency of label-free risk estimation]
\label{cor:risk-consistency}
Let $\widehat{\mathrm{TPR}}_a$, $\widehat{\mathrm{FPR}}_a$ be empirical estimates from $n_a$ source samples per group, and let $\widehat\pi_{\tgt,a}$ be a consistent prevalence estimator from $m_a$ unlabeled target samples.
Then the plug-in estimator obtained by replacing $(\mathrm{TPR}_a,\mathrm{FPR}_a,\pi_{\tgt,a},P_{\tgt}(A=a))$ with their empirical estimates is consistent: $\widehat\Risk_{\tgt}(\hat Y) \xrightarrow{p} \Risk_{\tgt}(\hat Y)$ as $\min_a(n_a, m_a) \to \infty$.
\end{corollary}

\paragraph{Finite-sample bounds and misspecification.}
Finite-sample bounds for DP gap estimation and a discussion of model misspecification are provided in Appendix~\ref{sec:appendix-finite-sample}.

\section{Experiments}
\label{sec:experiments}

We design experiments to validate our theoretical predictions and evaluate the effectiveness of TAP-GPPS.
We evaluate three questions: (i) whether equalized odds remains invariant under GPPS; (ii) whether demographic parity and PPV drift as predicted; and (iii) whether TAP-GPPS restores target-domain fairness with minimal utility loss.

\paragraph{Setups.}
We use a synthetic Gaussian model and three real-world datasets (Adult \citep{kohavi1996uci}, COMPAS \citep{angwin2016machine}, MEPS \citep{cohen2009medical}).
We construct GPPS source/target splits by varying group-conditional prevalences while preserving $P(X\mid Y,A)$; full preprocessing and split construction details are in Appendix~\ref{sec:appendix-experiments}.

\paragraph{Models, baselines, and metrics.}
We train calibrated logistic regression, XGBoost \citep{chen2016xgboost}, and a 2-layer MLP, and compare Source-only, TAP-GPPS (EM/BBSE), a No-correction DP baseline, and an oracle that tunes thresholds using labeled target validation data (Appendix~\ref{sec:appendix-experiments}).
We report target accuracy and gaps for equalized odds (EO), demographic parity (DP), and predictive parity (PPV); definitions are in Appendix~\ref{sec:appendix-experiments}.
For drift plots, we overlay theoretical predictions from Eqs.~\eqref{eq:ar-affine} and~\eqref{eq:ppv-formula}.

\subsection{Results}

\paragraph{Validation of invariance and drift.}
Figure~\ref{fig:drift-curves} shows fairness metrics as a function of the prevalence shift magnitude on synthetic data.
As predicted by Theorem~\ref{thm:eo-inv}, the EO gap remains approximately constant ($0.03 \pm 0.02$) across all shift levels, confirming structural invariance under GPPS.
In contrast, the DP gap varies from near zero at matched prevalences to over $0.2$ at maximum shift, following the linear relationship predicted by Proposition~\ref{prop:ar-affine}.
The PPV gap exhibits the nonlinear behavior predicted by Proposition~\ref{prop:ppv}.
Theoretical predictions (dashed lines) computed from source-estimated ROC quantities closely match empirical observations, validating our closed-form expressions.

\paragraph{TAP-GPPS effectiveness.}
Appendix Table~\ref{tab:main-results} reports target-domain performance on synthetic data with source prevalences $(0.3, 0.5)$ shifted to $(0.5, 0.3)$.
\emph{Critically, we use separate validation and test sets}: thresholds are selected on the validation set while final metrics are evaluated on a held-out test set, ensuring unbiased estimates.

Without adaptation, Source-only incurs a DP gap of $0.106$---a substantial fairness violation.
We compare against the \textbf{No-correction DP} baseline, which selects group thresholds to equalize acceptance rates on uncorrected posteriors.
This baseline achieves DP gap $0.023$, demonstrating that threshold adjustment alone (without posterior correction) can substantially reduce DP violations.

TAP-GPPS with either EM or BBSE achieves the lowest DP gap ($0.020$), slightly outperforming No-correction DP.
The improvement is modest because for DP specifically, threshold selection dominates the correction; however, TAP-GPPS provides additional benefits: (i) the corrected posteriors enable label-free risk estimation (Theorem~\ref{thm:risk-identifiability}), and (ii) the framework extends naturally to other fairness criteria.
The non-zero DP gap arises from finite-sample estimation error---as predicted by Theorem~\ref{thm:finite-sample-dp}---but is well within typical tolerance thresholds ($<0.05$).
This comes at approximately 2\% accuracy reduction compared to Source-only.
Notably, enforcing DP via group thresholds increases EO gap from $0.029$ to $\approx 0.15$, illustrating the well-known tradeoff between fairness criteria \citep{chouldechova2017fair,kleinberg2017inherent}.

\paragraph{Sample complexity and learning curves.}
Learning-curve results validating the $O(1/\sqrt{m})$ scaling in Theorem~\ref{thm:finite-sample-dp} are provided in Appendix~\ref{sec:appendix-experiments}.

\paragraph{Semi-synthetic Adult experiments.}
On the UCI Adult dataset with GPPS induced by resampling, we observe consistent patterns.
Under moderate shift (target prevalences $(0.15, 0.40)$), Source-only achieves $0.846$ accuracy with DP gap $0.189$, while TAP-GPPS achieves $0.821$ accuracy with DP gap $0.013$.
Under convergent shift ($(0.20, 0.25)$), where group prevalences become more similar, Source-only achieves $0.855$ accuracy with DP gap $0.149$, and TAP-GPPS achieves $0.850$ accuracy with DP gap $0.004$.
Under divergent shift ($(0.08, 0.45)$), the DP gap for Source-only increases to $0.216$, while TAP-GPPS reduces it to $0.006$ with accuracy $0.806$.

\paragraph{EM vs.\ BBSE prevalence estimation.}
Both prevalence estimation methods achieve similar performance in our experiments.
EM and BBSE yield nearly identical accuracy and fairness metrics, suggesting that with sufficient unlabeled target data ($m \geq 500$ per group), prevalence estimation is not the bottleneck.
We recommend EM when the source model is well-calibrated, and BBSE when robustness to model misspecification is desired.

\begin{figure}[t]
\centering
\includegraphics[width=\columnwidth]{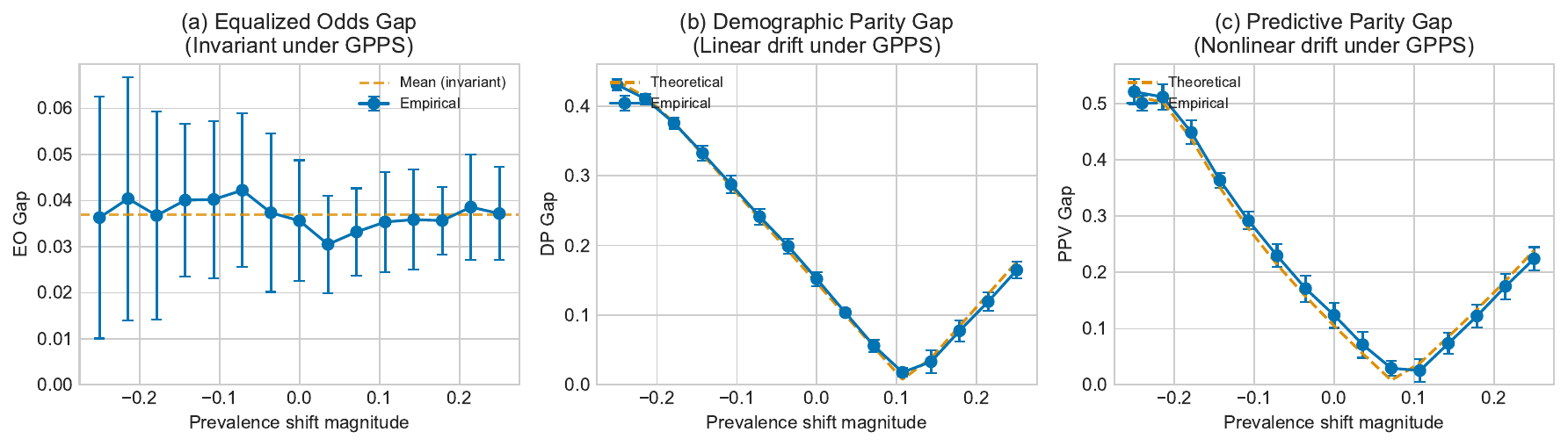}
\caption{Validation of invariance and drift predictions on synthetic data. \textbf{(a)} EO gap remains constant ($\approx 0.03$) across all prevalence shifts, confirming invariance under GPPS (Theorem~\ref{thm:eo-inv}). \textbf{(b)} DP gap varies linearly with prevalence shift, matching theoretical predictions (dashed line) from Proposition~\ref{prop:ar-affine}. \textbf{(c)} PPV gap varies nonlinearly with prevalence shift, following Proposition~\ref{prop:ppv}. Error bars show $\pm 1$ standard deviation over 5 random seeds.}
\label{fig:drift-curves}
\end{figure}

\section{Discussion}
\label{sec:discussion}

\paragraph{Practical implications.}
Under GPPS, separation-based criteria such as equalized odds and equality of opportunity transfer without adaptation (Theorem~\ref{thm:eo-inv}).
In contrast, demographic parity and predictive parity can drift with groupwise prevalences (Propositions~\ref{prop:ar-affine} and~\ref{prop:ppv}), so practitioners should monitor prevalence changes and deploy target-aware post-processing such as TAP-GPPS when these criteria are mandated.
The closed-form drift expressions (Eqs.~\ref{eq:ar-affine} and~\ref{eq:ppv-formula}) enable proactive monitoring by forecasting expected violations from estimated target prevalences.

\paragraph{Limitations and extensions.}
Our guarantees assume GPPS holds and that sensitive attributes are available at deployment for groupwise prevalence estimation and thresholding.
Developing methods and guarantees under approximate GPPS, and under restricted access to $A$, remains important future work; Appendix~\ref{sec:appendix-finite-sample} provides finite-sample bounds for the label-free estimators used by TAP-GPPS.
Extensions to multi-class outcomes and continuous sensitive attributes are natural and can be handled by applying the same correction and thresholding ideas within each class/stratum.

\section{Conclusion}
\label{sec:conclusion}

We have studied fairness under group-conditional prior probability shift, a distribution shift mechanism where label prevalence changes differently across demographic groups while the feature-generation process remains stable.
Our main theoretical contribution is a clean dichotomy: separation-based fairness criteria (equalized odds, equality of opportunity) are structurally invariant under GPPS, while acceptance-rate and predictive-value criteria (demographic parity, predictive parity) can drift substantially.
We provided closed-form expressions quantifying this drift and proposed TAP-GPPS, a modular algorithm that estimates target prevalences from unlabeled data and adapts thresholds to satisfy target-domain fairness constraints.
Experiments on synthetic and real-world datasets validate our theoretical predictions and demonstrate the effectiveness of TAP-GPPS in reducing fairness violations under GPPS.
By characterizing which fairness criteria are robust to prevalence shift and which are not, we hope to help practitioners make more informed decisions about criterion selection and deployment monitoring.

\section*{Broader Impact}
This paper addresses the robustness of algorithmic fairness under distribution shift, a topic with direct societal implications.
Our work could positively impact society by helping practitioners maintain fairness guarantees as deployment conditions change, reducing the risk of discrimination in automated decision-making systems.
However, we note potential risks: our framework assumes access to sensitive attributes, raising privacy considerations; the choice of fairness criterion remains a normative decision that our technical results do not resolve; and practitioners should not use prevalence-shift explanations to dismiss legitimate fairness concerns that arise from other sources.
We encourage responsible application of these methods in conjunction with broader fairness auditing and stakeholder engagement.

\bibliographystyle{icml2026}
\bibliography{ref3}

\appendix
\onecolumn
\section{Related Work}
\label{sec:appendix-related}

\paragraph{Statistical group fairness.}
The study of algorithmic fairness has produced a rich taxonomy of statistical criteria.
Demographic parity requires equal positive prediction rates across groups \citep{calders2009building,feldman2015certifying}.
Equalized odds and equality of opportunity require equal true positive and/or false positive rates across groups \citep{hardt2016equality}.
Predictive parity (calibration within groups) requires equal positive predictive values \citep{chouldechova2017fair,kleinberg2017inherent}.
\citet{chouldechova2017fair} and \citet{kleinberg2017inherent} establish that these criteria are generally incompatible except in degenerate cases, motivating the study of which criteria to prioritize in different contexts.
Our work contributes to this understanding by characterizing which criteria remain valid under a particular---and practically relevant---form of distribution shift.

\paragraph{Post-processing for fairness.}
A common approach to achieving fairness is to post-process a trained classifier's outputs.
\citet{hardt2016equality} propose threshold adjustment to achieve equalized odds.
\citet{pleiss2017calibration} study calibration and fairness, showing how to achieve calibrated predictions within groups.
\citet{corbett2017algorithmic} analyze fairness interventions in the context of criminal justice.
Our TAP-GPPS algorithm extends the post-processing paradigm to handle distribution shift by incorporating target-domain prevalence estimates before threshold selection.

\paragraph{Fairness under distribution shift.}
Recent work has begun characterizing how fairness guarantees transfer across distributions.
\citet{giguere2022fairness} study demographic shift and provide high-confidence fairness guarantees that hold under changes in demographic marginals.
\citet{chen2022fairness} study transferability of group fairness under bounded shifts, deriving bounds for demographic parity and equalized odds under several shift classes.
\citet{maity2021enforcing} characterize when enforcing risk-based fairness improves or harms target performance under subpopulation shift.
\citet{rezaei2021robust} propose robust estimation for fair prediction under covariate shift.
\citet{singh2021fairness} study fairness in the presence of selection bias.
\citet{schrouff2022maintaining} examine maintaining fairness across distribution shift in clinical applications.
Our work complements this literature by isolating groupwise prior probability shift---a shift mechanism with particularly clean structure---and providing exact invariance results rather than bounds.

\paragraph{Label shift and quantification.}
Under label shift (prior probability shift), the label marginal changes while $\Prb(X\mid Y)$ remains stable \citep{storkey2009training}.
\citet{saerens2002adjusting} proposed an EM procedure to adjust classifier outputs to new priors.
\citet{lipton2018detecting} introduced black-box shift estimation (BBSE), which estimates target label proportions using a confusion matrix computed on source data.
The quantification literature studies prevalence estimation under prior shift, with contributions on optimal estimators \citep{vaz2019quantification}, minimax rates \citep{iyer2014maximum}, and practical algorithms \citep{forman2005counting,bella2010quantification}.
We extend these techniques to the group-conditional setting, applying them within each demographic group.

\paragraph{Fairness under prior probability shift.}
Most directly related to our work, \citet{biswas2019proportional} and \citet{biswas2020ensuring} explicitly study fairness under prior probability shifts within demographic subgroups.
They propose new fairness criteria (e.g., Proportional Equality) and algorithms tailored to this setting.
Our contribution is complementary: rather than proposing new fairness definitions, we characterize invariance and drift of \emph{standard} criteria under groupwise prior shift, providing practitioners with actionable guidance on which existing criteria to trust under this form of shift and how to adapt when necessary.

\section{Extended Experimental Details}
\label{sec:appendix-experiments}

This appendix provides experimental details deferred from the main text, including dataset construction under GPPS, model hyperparameters, baseline definitions, metric definitions, and learning-curve results.

\paragraph{Datasets and GPPS construction.}
\emph{Synthetic Gaussian model.} We generate $X\in\R^2$ with two groups $A\in\{0,1\}$. For each group $a$ and label $y$, we sample
$X\mid(Y=y,A=a)\sim \mathcal{N}(\mu_{y,a},\Sigma)$ with shared covariance $\Sigma$ and group-specific means $\mu_{y,a}$.
GPPS holds exactly when we vary $\pi_{s,a}$ across domains while keeping $(\mu_{y,a},\Sigma)$ fixed; in the source we use $(\pi_{\src,0},\pi_{\src,1})=(0.3,0.5)$ and in the target we vary prevalences across a grid spanning $(0.1,0.7)$ to $(0.7,0.1)$.

\emph{Adult \citep{kohavi1996uci}.} We use the UCI Adult dataset ($n\approx 48{,}000$) with gender as the sensitive attribute and income as the label.
To induce GPPS, we resample within each $(A,Y)$ stratum to match desired target prevalences while preserving $P(X\mid Y,A)$.

\emph{COMPAS \citep{angwin2016machine}.} We use the ProPublica COMPAS dataset ($n\approx 6{,}000$) with race (White vs.\ non-White) as the sensitive attribute and recidivism as the label.
Following \citet{biswas2020ensuring}, we form source/target splits that simulate temporal groupwise prevalence drift.

\emph{MEPS \citep{cohen2009medical}.} We use the Medical Expenditure Panel Survey with race as the sensitive attribute and predict high utilization ($Y=1$ if utilization exceeds the median), leveraging natural prevalence variation across years to form GPPS-like source/target splits.

\paragraph{Models, calibration, and baselines.}
We train three base model families: logistic regression with L2 regularization (selected via cross-validation), gradient boosted trees (XGBoost \citep{chen2016xgboost}), and a 2-layer MLP trained with Adam.
All models are calibrated on a held-out source validation set using temperature scaling \citep{guo2017calibration}.
We compare: \textbf{Source-only} (no target adaptation), \textbf{TAP-GPPS} with EM or BBSE prevalence estimation, \textbf{No-correction DP} (thresholding for DP without posterior correction), and an \textbf{Oracle} that tunes thresholds using labeled target validation data.

\paragraph{Prevalence estimation.}
For EM \citep{saerens2002adjusting}, we run 50 iterations with convergence threshold $10^{-6}$.
For BBSE \citep{lipton2018detecting}, we estimate a groupwise confusion matrix on source validation data and invert it to estimate target prevalences.

\paragraph{Metrics.}
We report target accuracy and fairness gaps for EO, DP, and PPV.
For two groups with (possibly) group-specific thresholds $t_a$, the gaps are:
\begin{align*}
\gap_{\mathrm{EO}} &:= |\mathrm{TPR}_{0}(t_0)-\mathrm{TPR}_{1}(t_1)| + |\mathrm{FPR}_{0}(t_0)-\mathrm{FPR}_{1}(t_1)|,\\
\gap_{\mathrm{DP}} &:= |\mathrm{AR}_{\tgt,0}(t_0)-\mathrm{AR}_{\tgt,1}(t_1)|,\qquad
\gap_{\mathrm{PPV}} := |\mathrm{PPV}_{\tgt,0}(t_0)-\mathrm{PPV}_{\tgt,1}(t_1)|.
\end{align*}
For drift plots, we vary target prevalences and overlay the closed-form predictions from Eqs.~\eqref{eq:ar-affine} and~\eqref{eq:ppv-formula}.

\begin{table}[t]
\centering
\caption{Target-domain performance under GPPS on synthetic data. Results evaluated on held-out test set (separate from threshold selection). Lower is better for gaps; higher is better for accuracy. Results averaged over 10 random seeds ($\pm$ standard deviation). TAP-GPPS achieves the lowest DP gap among all methods.}
\label{tab:main-results}
\resizebox{\columnwidth}{!}{%
\begin{tabular}{lcccc}
\toprule
Method & Accuracy & DP Gap & EO Gap & PPV Gap \\
\midrule
Source-only & $0.827 \pm 0.006$ & $0.106 \pm 0.015$ & $0.029 \pm 0.013$ & $0.144 \pm 0.030$ \\
No-correction DP & $0.810 \pm 0.006$ & $0.023 \pm 0.016$ & $0.149 \pm 0.026$ & $0.235 \pm 0.029$ \\
TAP-GPPS (EM) & $0.810 \pm 0.005$ & $\mathbf{0.020 \pm 0.013}$ & $0.152 \pm 0.022$ & $0.229 \pm 0.033$ \\
TAP-GPPS (BBSE) & $0.811 \pm 0.005$ & $\mathbf{0.020 \pm 0.013}$ & $0.154 \pm 0.023$ & $0.228 \pm 0.033$ \\
Oracle & $0.818 \pm 0.005$ & $0.037 \pm 0.025$ & $0.104 \pm 0.029$ & $0.196 \pm 0.034$ \\
\bottomrule
\end{tabular}}
\end{table}

\begin{figure}[t]
\centering
\includegraphics[width=0.85\textwidth]{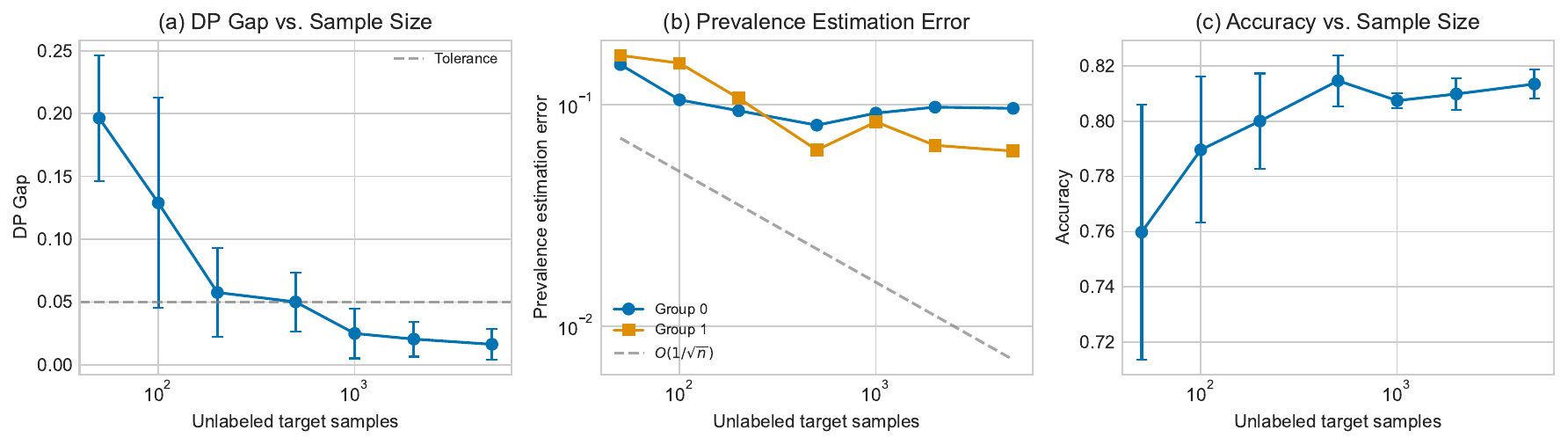}
\caption{Sample complexity of TAP-GPPS on synthetic data. \textbf{(a)} DP gap decreases with unlabeled sample size, crossing $0.05$ with roughly $500$ target samples per group. \textbf{(b)} Prevalence estimation error follows the theoretical $O(1/\sqrt{m})$ rate. \textbf{(c)} Accuracy stabilizes quickly.}
\label{fig:learning-curves}
\end{figure}

\section{Complete Proofs}
\label{sec:appendix-proofs}

\subsection{Proof of Lemma~\ref{lem:score-inv}}
\begin{proof}
Fix $a\in\cA$, $y\in\{0,1\}$.
Under GPPS, the conditional distribution of $X$ given $(Y=y,A=a)$ is the same in source and target:
$P_{\tgt}(X\mid Y=y,A=a)=P_{\src}(X\mid Y=y,A=a)$.
Let $f_a(x):=f(x,a)$.
For any measurable $B\subseteq\R$,
\begin{align*}
\Prb_{\tgt}(f(X,A)\in B\mid Y=y,A=a)
&= \Prb_{\tgt}(f_a(X)\in B\mid Y=y,A=a)\\
&= \int \1\{f_a(x)\in B\}\, dP_{\tgt}(x\mid Y=y,A=a)\\
&= \int \1\{f_a(x)\in B\}\, dP_{\src}(x\mid Y=y,A=a)\\
&= \Prb_{\src}(f(X,A)\in B\mid Y=y,A=a),
\end{align*}
which proves the claim.
\end{proof}

\subsection{Proof of Corollary~\ref{cor:roc-inv}}
\begin{proof}
Apply Lemma~\ref{lem:score-inv} with $B=[t,\infty)$ and note that
$\Prb_s(f(X,A)\in[t,\infty)\mid Y=y,A=a)=\Prb_s(f(X,a)\ge t\mid Y=y,A=a)$.
This gives invariance of both TPR (for $y=1$) and FPR (for $y=0$).
\end{proof}

\subsection{Proof of Theorem~\ref{thm:eo-inv}}
\begin{proof}
Fix $a\in\cA$ and $y\in\{0,1\}$.
Let $t_a:=\tau(a)$.
Then
\begin{align*}
\Prb_{\tgt}(\hat Y=1\mid Y=y,A=a)
&= \Prb_{\tgt}(f(X,a)\ge t_a\mid Y=y,A=a)\\
&= \Prb_{\src}(f(X,a)\ge t_a\mid Y=y,A=a)\qquad\text{(by Corollary~\ref{cor:roc-inv})}\\
&= \Prb_{\src}(\hat Y=1\mid Y=y,A=a).
\end{align*}
Thus all group-conditional error rates transfer.
If EO holds on source, then for any $a,a',y$,
\begin{align*}
\Prb_{\tgt}(\hat Y=1\mid Y=y,A=a)
&= \Prb_{\src}(\hat Y=1\mid Y=y,A=a)\\
&= \Prb_{\src}(\hat Y=1\mid Y=y,A=a')\\
&= \Prb_{\tgt}(\hat Y=1\mid Y=y,A=a').
\end{align*}
Hence EO holds on target.
\end{proof}

\subsection{Proof of Proposition~\ref{prop:ar-affine}}
\begin{proof}
Fix group $a$ and threshold $t$.
By the law of total probability conditional on $A=a$,
\begin{align*}
\mathrm{AR}_{s,a}(t)
&= \Prb_s(f(X,a)\ge t\mid A=a)\\
&= \Prb_s(f(X,a)\ge t\mid Y=1,A=a)\Prb_s(Y=1\mid A=a)\\
&\quad + \Prb_s(f(X,a)\ge t\mid Y=0,A=a)\Prb_s(Y=0\mid A=a)\\
&= \mathrm{TPR}_{s,a}(t)\,\pi_{s,a} + \mathrm{FPR}_{s,a}(t)\,(1-\pi_{s,a}).
\end{align*}
Under GPPS, Corollary~\ref{cor:roc-inv} gives $\mathrm{TPR}_{s,a}(t)=\mathrm{TPR}_{a}(t)$ and $\mathrm{FPR}_{s,a}(t)=\mathrm{FPR}_{a}(t)$, yielding Eq.~\eqref{eq:ar-affine}.
\end{proof}

\subsection{Proof of Proposition~\ref{prop:ppv}}
\begin{proof}
Fix $a$ and $t$.
By Bayes' rule,
\begin{align*}
\mathrm{PPV}_{s,a}(t)
&= \Prb_s(Y=1\mid \hat Y=1,A=a)\\
&= \frac{\Prb_s(\hat Y=1\mid Y=1,A=a)\Prb_s(Y=1\mid A=a)}{\Prb_s(\hat Y=1\mid A=a)}.
\end{align*}
The numerator equals $\mathrm{TPR}_{a}(t)\pi_{s,a}$.
The denominator is $\mathrm{AR}_{s,a}(t)$, which by Proposition~\ref{prop:ar-affine} equals $\pi_{s,a}\mathrm{TPR}_{a}(t)+(1-\pi_{s,a})\mathrm{FPR}_{a}(t)$.
Substituting yields Eq.~\eqref{eq:ppv-formula}.
\end{proof}

\subsection{Proof of Proposition~\ref{prop:posterior-correction}}
\begin{proof}
Fix group $a$.
Under GPPS, $P_{\tgt}(x\mid y,a)=P_{\src}(x\mid y,a)$.
By Bayes' rule,
\begin{equation*}
  p_{\tgt}(y\mid x,a) = \frac{p_{\tgt}(x\mid y,a)\,p_{\tgt}(y\mid a)}{p_{\tgt}(x\mid a)}.
\end{equation*}
For binary $y\in\{0,1\}$, the normalization gives
\begin{equation*}
  p_{\tgt}(1\mid x,a) = \frac{p_{\tgt}(x\mid 1,a)\pi_{\tgt,a}}{p_{\tgt}(x\mid 1,a)\pi_{\tgt,a} + p_{\tgt}(x\mid 0,a)(1-\pi_{\tgt,a})}.
\end{equation*}
Replace $p_{\tgt}(x\mid y,a)$ by $p_{\src}(x\mid y,a)$.
Also write the source posterior as
\begin{equation*}
  p_{\src}(1\mid x,a) = \frac{p_{\src}(x\mid 1,a)\pi_{\src,a}}{p_{\src}(x\mid 1,a)\pi_{\src,a} + p_{\src}(x\mid 0,a)(1-\pi_{\src,a})}.
\end{equation*}
Algebraic rearrangement yields Eq.~\eqref{eq:posterior-correction}. The log-odds form follows by taking $\operatorname{logit}(p):=\log\frac{p}{1-p}$.
\end{proof}

\subsection{Proof of Lemma~\ref{lem:monotone-ar}}
\begin{proof}
For each $t$, define the event $E_t:=\{p_{\tgt}(1\mid X,a)\ge t\}$.
If $t_1\le t_2$ then $E_{t_2}\subseteq E_{t_1}$, so
$\Prb_{\tgt}(E_{t_2}\mid A=a)\le \Prb_{\tgt}(E_{t_1}\mid A=a)$.
This proves monotonicity.
Continuity follows if $p_{\tgt}(1\mid X,a)$ has a density: then $\Prb(p_{\tgt}(1\mid X,a)=t\mid A=a)=0$, and the survival function is continuous.
\end{proof}

\subsection{Proof of Proposition~\ref{prop:dp-threshold}}
\begin{proof}
Under Lemma~\ref{lem:monotone-ar} and continuity, the function $t\mapsto \mathrm{AR}_{\tgt,a}(t)$ is continuous and maps $[0,1]$ to $[0,1]$ with $\mathrm{AR}_{\tgt,a}(0)=1$ and $\mathrm{AR}_{\tgt,a}(1)=\Prb_{\tgt}(p_{\tgt}(1\mid X,a)=1\mid A=a)$ (typically $0$).
For any $\gamma$ in the attainable range, the intermediate value theorem yields existence of $t_a$ such that $\mathrm{AR}_{\tgt,a}(t_a)=\gamma$.
Monotonicity yields that bisection finds $t_a$ to accuracy $\epsilon$ in $O(\log(1/\epsilon))$ evaluations.
In practice, $\mathrm{AR}_{\tgt,a}(t)$ is estimated empirically on unlabeled target data using corrected posteriors based on $\widehat\pi_{\tgt,a}$.
\end{proof}

\subsection{Proof of Theorem~\ref{thm:dp-impossibility} (Shift-robust impossibility for DP)}
\begin{proof}
Suppose $\hat Y$ satisfies demographic parity under prevalences $(\pi_0, \pi_1)$ and $(\pi'_0, \pi'_1)$.
By Proposition~\ref{prop:ar-affine}, the acceptance rate for group $a$ under prevalence $\pi_a$ is
\begin{equation*}
\mathrm{AR}_a(\pi_a) = \pi_a \mathrm{TPR}_a + (1-\pi_a) \mathrm{FPR}_a = \mathrm{FPR}_a + \pi_a(\mathrm{TPR}_a - \mathrm{FPR}_a).
\end{equation*}
DP requires $\mathrm{AR}_0 = \mathrm{AR}_1$ in both regimes.
Under $(\pi_0, \pi_1)$:
\begin{equation*}
\mathrm{FPR}_0 + \pi_0 \Delta_0 = \mathrm{FPR}_1 + \pi_1 \Delta_1,
\end{equation*}
where $\Delta_a := \mathrm{TPR}_a - \mathrm{FPR}_a$.
Under $(\pi'_0, \pi'_1)$:
\begin{equation*}
\mathrm{FPR}_0 + \pi'_0 \Delta_0 = \mathrm{FPR}_1 + \pi'_1 \Delta_1.
\end{equation*}
Subtracting yields:
\begin{equation*}
(\pi'_0 - \pi_0)\Delta_0 = (\pi'_1 - \pi_1)\Delta_1.
\end{equation*}
If $\pi_a \ne \pi'_a$ for some $a$, then either:
(i) $\Delta_0 = \Delta_1 = 0$ (uninformative classifier), or
(ii) both $\Delta_0, \Delta_1 \ne 0$ and the constraint $(\pi'_0 - \pi_0)/(\pi'_1 - \pi_1) = \Delta_1/\Delta_0$ holds.
\end{proof}

\subsection{Proof of Theorem~\ref{thm:ppv-impossibility} (Shift-robust impossibility for PPV)}
\begin{proof}
By Proposition~\ref{prop:ppv}, PPV for group $a$ under prevalence $\pi_a$ is
\begin{equation*}
\mathrm{PPV}_a(\pi_a) = \frac{\pi_a \mathrm{TPR}_a}{\pi_a \mathrm{TPR}_a + (1-\pi_a)\mathrm{FPR}_a}.
\end{equation*}
Predictive parity requires $\mathrm{PPV}_0(\pi_0) = \mathrm{PPV}_1(\pi_1)$ under both regimes.
Cross-multiplying the PPV equality gives:
\begin{multline*}
\pi_0 \mathrm{TPR}_0 [\pi_1 \mathrm{TPR}_1 + (1-\pi_1)\mathrm{FPR}_1] = \pi_1 \mathrm{TPR}_1 [\pi_0 \mathrm{TPR}_0 + (1-\pi_0)\mathrm{FPR}_0].
\end{multline*}
Simplifying:
\begin{equation*}
\pi_0 \mathrm{TPR}_0 (1-\pi_1)\mathrm{FPR}_1 = \pi_1 \mathrm{TPR}_1 (1-\pi_0)\mathrm{FPR}_0.
\end{equation*}
This defines a hypersurface in $(\pi_0, \pi_1)$ space.
For generic classifier parameters ($\mathrm{TPR}_a, \mathrm{FPR}_a$ not satisfying special relationships), two distinct points $(\pi_0, \pi_1)$ and $(\pi'_0, \pi'_1)$ on this surface have measure zero.
The special case where $\mathrm{TPR}_0/\mathrm{FPR}_0 = \mathrm{TPR}_1/\mathrm{FPR}_1$ makes the constraint degenerate, allowing multiple solutions.
\end{proof}

\subsection{Proof of Theorem~\ref{thm:risk-identifiability} (Label-free risk identifiability)}
\begin{proof}
The target risk under 0--1 loss is
\begin{align*}
\Risk_{\tgt}(\hat Y) &= \E_{\tgt}[\1\{\hat Y \ne Y\}] \\
&= \sum_{a} P_{\tgt}(A=a) \E_{\tgt}[\1\{\hat Y \ne Y\} \mid A=a].
\end{align*}
Conditioning on $Y$:
\begin{align*}
&\E_{\tgt}[\1\{\hat Y \ne Y\} \mid A=a] \\
&= \pi_{\tgt,a} \Prb_{\tgt}(\hat Y = 0 \mid Y=1, A=a) + (1-\pi_{\tgt,a}) \Prb_{\tgt}(\hat Y = 1 \mid Y=0, A=a) \\
&= \pi_{\tgt,a}(1 - \mathrm{TPR}_{\tgt,a}) + (1-\pi_{\tgt,a})\mathrm{FPR}_{\tgt,a}.
\end{align*}
By Corollary~\ref{cor:roc-inv}, $\mathrm{TPR}_{\tgt,a} = \mathrm{TPR}_{\src,a}$ and $\mathrm{FPR}_{\tgt,a} = \mathrm{FPR}_{\src,a}$ under GPPS.
Thus the risk depends only on:
(i) $\mathrm{TPR}_a, \mathrm{FPR}_a$, estimable from labeled source data;
(ii) $\pi_{\tgt,a}$, estimable from unlabeled target data;
(iii) $P_{\tgt}(A=a)$, directly observable from target data.
\end{proof}

\subsection{Finite-sample DP bound}
\label{sec:appendix-finite-sample}

\begin{theorem}[Finite-sample bound for DP gap estimation]
\label{thm:finite-sample-dp}
Let $\widehat\pi_{\tgt,a}$ be the BBSE prevalence estimator with $m_a$ unlabeled target samples in group $a$, and let $\widehat{\mathrm{TPR}}_a$, $\widehat{\mathrm{FPR}}_a$ be empirical estimates from $n_a$ labeled source samples.
Assume the classifier has $\mathrm{TPR}_a - \mathrm{FPR}_a \ge \kappa > 0$ for all $a$.
Then with probability at least $1 - \delta$,
\begin{multline}
\label{eq:dp-bound}
  \big| \widehat{\gap}_{\mathrm{DP}} - \gap_{\mathrm{DP}}(\tgt) \big| \le 
  \sum_{a \in \cA} \left( \frac{C_1}{\kappa\sqrt{m_a}} + \frac{C_2}{\sqrt{n_a}} \right) \sqrt{\log(|\cA|/\delta)},
\end{multline}
where $C_1, C_2$ are universal constants.
\end{theorem}

\begin{proof}[Proof sketch]
The DP gap estimator is
\begin{equation*}
\widehat{\gap}_{\mathrm{DP}} = |\widehat{\mathrm{AR}}_0 - \widehat{\mathrm{AR}}_1|
\end{equation*}
where $\widehat{\mathrm{AR}}_a = \widehat\pi_{\tgt,a} \widehat{\mathrm{TPR}}_a + (1-\widehat\pi_{\tgt,a})\widehat{\mathrm{FPR}}_a$.
By the triangle inequality:
\begin{equation*}
|\widehat{\gap}_{\mathrm{DP}} - \gap_{\mathrm{DP}}| \le |\widehat{\mathrm{AR}}_0 - \mathrm{AR}_0| + |\widehat{\mathrm{AR}}_1 - \mathrm{AR}_1|.
\end{equation*}
For each group, the error decomposes as:
\begin{align*}
|\widehat{\mathrm{AR}}_a - \mathrm{AR}_a| &\le |\widehat\pi_{\tgt,a} - \pi_{\tgt,a}| \cdot |\mathrm{TPR}_a - \mathrm{FPR}_a| \\
&\quad + \pi_{\tgt,a} |\widehat{\mathrm{TPR}}_a - \mathrm{TPR}_a| + (1-\pi_{\tgt,a})|\widehat{\mathrm{FPR}}_a - \mathrm{FPR}_a|.
\end{align*}
Standard concentration bounds give $|\widehat\pi_{\tgt,a} - \pi_{\tgt,a}| = O(1/(\kappa\sqrt{m_a}))$ for BBSE when $\kappa = \mathrm{TPR}_a - \mathrm{FPR}_a > 0$ \citep{lipton2018detecting}, and $|\widehat{\mathrm{TPR}}_a - \mathrm{TPR}_a|, |\widehat{\mathrm{FPR}}_a - \mathrm{FPR}_a| = O(1/\sqrt{n_a})$ by Hoeffding.
Combining with a union bound over groups yields the stated result.
\end{proof}

\end{document}